\documentclass{article}

\usepackage{times}
\usepackage{graphicx}
\usepackage{subfigure}
\usepackage{framed}

\usepackage{natbib}

\usepackage{algorithm}
\usepackage{algorithmic}

\usepackage{amsmath}
\usepackage{amsthm}
\newtheorem{thm}{Theorem}
\newtheorem{lem}{Lemma}
\allowdisplaybreaks

\usepackage[accepted]{icml2014}

\newlength{\tcwidth}
\icmltitlerunning{Finito: A Faster, Permutable Incremental Gradient Method for Big Data Problems}

\begin{document} 

\setlength{\abovedisplayskip}{2mm}
\setlength{\belowdisplayskip}{2mm}

\twocolumn[
\icmltitle{Finito: A Faster, Permutable Incremental Gradient Method for Big Data Problems}

\icmlauthor{Aaron J. Defazio}{aaron.defazio@anu.edu.au}
\icmlauthor{Tib\'erio S. Caetano}{tiberio.caetano@nicta.com.au}
\icmlauthor{Justin Domke}{justin.domke@nicta.com.au}
\icmladdress{NICTA and Australian National University}

\icmlkeywords{incremental gradient, big data}

\vskip 0.3in
]

\begin{abstract}
Recent advances in optimization theory have shown that smooth strongly convex
finite sums can be minimized faster than by treating them as a black box "batch" problem.
In this work we introduce a new method in this class with a theoretical convergence 
rate four times faster than existing methods, for sums with sufficiently many terms.
This method is also amendable to a sampling without replacement scheme that in practice
gives further speed-ups. We give empirical results showing state of the 
art performance.
\end{abstract}

\section{Introduction}
\label{sec:introduction}

Many recent advances in the theory and practice of numerical optimization have come 
from the recognition and exploitation of structure. Perhaps the most common structure
is that of finite sums. In machine learning when applying empirical risk minimization 
we almost always end up with an optimization problem involving the minimization 
of a sum with one term per data point.

The recently developed SAG algorithm \citep{SAG} has shown that even with this simple form
of structure, as long as we have sufficiently many data points we are able to do significantly
better than black-box optimization techniques in expectation for smooth strongly convex problems.
In practical terms the difference is often a factor of 10 or more.

The requirement of sufficiently large datasets is fundamental to these methods. We describe the
precise form of this as the big data condition. Essentially, it is the requirement that the amount
of data is on the same order as the condition number of the problem.
The strong convexity requirement is not as onerous. Strong convexity holds in the common case where a quadratic
regularizer is used together with a convex loss.

The SAG method and the Finito method we describe in this work are similar in their form to 
stochastic gradient descent methods, but with one crucial difference: They store additional information
about each data point during optimization. Essentially, when they revisit a data point, they do not treat it as a 
novel piece of information every time. 

Methods for the minimization of finite sums have classically been known as Incremental gradient methods 
\citep{bertsekas-incremental}. The proof techniques used in SAG differ fundamentally from those
used on other incremental gradient methods though. The difference hinges on the requirement that data be accessed
in a randomized order. SAG does not work when data is accessed sequentially each epoch, so any proof
technique which shows even non-divergence for sequential access cannot be applied.

A remarkable property of Finito is the tightness of the theoretical bounds compared to
the practical performance of the algorithm. The practical convergence rate seen is at most twice as good as
the theoretically predicted rate. This sets it apart from methods such as LBFGS where the empirical performance is
often much better than the relatively weak theoretical convergence rates would suggest.

The lack of tuning required also sets Finito apart from stochastic gradient descent (SGD). In order to get good performance
out of SGD, substantial laborious tuning of multiple constants has traditionally been required. A multitude of 
heuristics have been developed to help choose these constants, or adapt them as the method progresses. Such heuristics
are more complex than Finito, and do not have the same theoretical backing. SGD has application outside of convex
problems of course, and we do not propose that Finito will replace SGD in those settings. Even on strongly convex
 problems SGD does not exhibit linear convergence like Finito does.

There are many similarities between SAG, Finto and stochastic dual coordinate descent (SDCA) methods \cite{SDCA}. SDCA
is only applicable to linear predictors. When it can be applied, it has linear convergence with theoretical rates 
similar to SAG and Finito.

\section{Algorithm}
\label{sec:algorithm}

We consider differentiable convex functions of the form
\[
f(w)=\frac{1}{n}\sum^{n}_{i=1}f_{i}(w).
\]

We assume that each $f_{i}$ has Lipschitz continuous gradients with constant $L$ and is
strongly convex with constant $s$. Clearly if we allow $n=1$, virtually all smooth, strongly
convex problems are included. So instead, we will restrict ourselves to problems satisfying the
\emph{big data} condition.

\textbf{Big data condition:} Functions of the above form satisfy the big data condition with constant $\beta$ if
\[
n \geq \beta \frac{L}{s}
\]
Typical values of $\beta$ are $1$-$8$. In plain language, we are considering problems where the amount of data
is of the same order as the condition number ($L/s$) of the problem.

\subsection{Additional Notation}
We superscript with 
$(k)$ to denote the value of the scripted quantity at iteration $k$.
We omit the $n$ superscript on summations, and subscript with $i$ with the implication that indexing starts at $1$.
When we use separate arguments for each $f_{i}$, we denote them $\phi_{i}$.
Let $\bar{\phi}^{(k)}$
denote the average $\bar{\phi}^{(k)}=\frac{1}{n}\sum^{n}_{i}\phi^{(k)}_{i}$. Our step length constant, which depends
on $\beta$, is denoted $\alpha$. We use angle bracket notation for dot products $\langle \cdot , \cdot \rangle$.

\subsection{The Finito algorithm}
We start with a table of known $\phi_{i}^{(0)}$ values, and
a table of known gradients $f_{i}^{\prime}(\phi_{i}^{(0)})$, for each $i$. We will update these two tables during the
course of the algorithm.
The step for iteration $k$, is as follows:
\begin{framed} 
\begin{enumerate}
  \setlength{\itemsep}{1pt}
\item Update $w$ using the step:
\[
w^{(k)}=\bar{\phi}^{(k)}-\frac{1}{\alpha sn}\sum_{i}f_{i}^{\prime}(\phi_{i}^{(k)}).
\]
\item Pick an index $j$ uniformly at random, or using without-replacement sampling as discussed in Section \ref{sec:randomness}.
\item Set $\phi_{j}^{(k+1)}=w^{(k)}$ in the table and leave the other variables the
same ($\phi_{i}^{(k+1)}=\phi_{i}^{(k)}$ for $i\neq j$).
\item Calculate and store $f_{j}^{\prime}(\phi_{j}^{(k+1)})$ in the table. 
\end{enumerate}
\vskip -5mm
\end{framed}
\vskip -2mm
Our main theoretical result is a convergence rate proof for this method.
\begin{thm}
When the big data condition holds with $\beta=2$, $\alpha=2$ may be used. In that setting, if we have initialized all $\phi^{(0)}_{i}$ the same, the convergence rate is:
\[
E\left[f(\bar{\phi}^{(k)})\right]-f(w^{*}) \leq \frac{3}{4s}\left(1-\frac{1}{2n}\right)^{k}\left\Vert f^{\prime}(\bar{\phi}^{(0)})\right\Vert ^{2}.
\]
\end{thm}
See Section \ref{sec:proof} for the proof. In contrast, SAG achieves a $\left(1-\frac{1}{8n}\right)$ rate when $\beta=2$. Note that on a per epoch basis, the Finito rate is $\left(1-\frac{1}{2n}\right)^{n} \approx \exp(-1/2) = 0.606$.  To put that into context, 10 epochs will see the error bound reduced by more than 148x.

One notable feature of our method is the fixed step size. In typical machine learning problems the strong convexity constant is given by the strength constant of the quadratic regularizer used. Since this is a known quantity, as long as the big data condition holds $\alpha=2$ may be used without any tuning or adjustment of Finito required. This lack of tuning is a major feature of Finito.

In cases where the big data condition does not hold, we conjecture that the step size must be reduced proportionally to the violation of the big data condition. In practice, the most effective step size can be found by testing a number of step sizes, as is usually done with other stochastic optimisation methods.

A simple way of satisfying the big data condition is to duplicate your data enough times so then holds. This is not as effective in practice as just changing the step size, and of course it uses more memory. However it does fall within the current theory.

Another difference compared to the SAG method is that we store both gradients and points $\phi_{i}$. We do not actually need twice as much memory however as they can be stored summed together. In particular we store the quantities $p_{i} = f_{i}^{\prime}(\phi_{i}) - \alpha s \phi_{i}$, and use the update rule $w=-\frac{1}{\alpha s n}\sum_{i} p_{i}$. This trick does not work when step lengths are adjusted during optimization however. The storage of $\phi_{i}$ is also a disadvantage when the gradients $f_{i}^{\prime}(\phi_{i})$ are sparse but $\phi_{i}$ are not sparse, as it can cause significant additional memory usage. We do not recommend the usage of Finito when gradients are sparse.

The SAG algorithm differs from Finito only in the $w$ update and step lengths:
\[
w^{(k)} = w^{(k-1)} -\frac{1}{16 L n}\sum_{i}f_{i}^{\prime}(\phi_{i}^{(k)}).
\]

\section{Randomness is key}
\label{sec:randomness}
By far the most interesting aspect of the SAG and Finito methods is the random choice of index at each iteration. We are not in an online setting, so there is no inherent randomness in the problem. Yet it seems that a randomized method is required. Neither method works in practice when the same ordering is used each pass, or in fact with any non-random access scheme we have tried. It is hard to emphasize enough the importance of randomness here. The technique of pre-permuting the data, then doing in order passes after that, also does not work. Reducing the step size in SAG or Finito by 1 or 2 orders of magnitude does not fix the convergence issues either.

Other methods, such as standard SGD, have been noted by various authors to exhibit speed-ups when random sampling is used instead of in order passes, but the differences are not as extreme as convergence v.s. non-convergence. Perhaps the most similar problem is that of coordinate descent on smooth convex functions. Coordinate descent cannot diverge when non-random orderings are used, but convergence rates are substantially worse in the non-randomized setting (\citealt{nes-coord}, \citealt{random-coordinate}).

Reducing the step size $\alpha$ by a much larger amount, namely by a factor of $n$, does allow for non-randomized orderings to be used. This gives an \emph{extremely} slow method however. This is the case covered by the MISO \citep{miso}. A similar reduction in step size gives convergence under non-randomized orderings for SAG also. Convergence rates for incremental sub-gradient methods with a variety of orderings appear in the literature also \citep{incremental-subgradient}.

\subsection*{Sampling without replacement is much faster}
\label{sec:permuted}
Other sampling schemes, such as sampling without replacement, should be considered. In detail, we mean the case where each "pass" over the data is a set of sampling without replacement steps, which continue until no data remains, after which another "pass" starts afresh. We call this the permuted case for simplicity, as it is the same as re-permuting the data after each pass. In practice, this approach does not give any speedup with SAG, however it works spectacularly well with Finito. We see speedups of up to a factor of two using this approach. This is one of the major differences in practice between SAG and Finito. We should note that we have no theory to support this case however. We are not aware of any analysis that proves faster convergence rates of any optimization method under a sampling without replacement scheme. An interesting discussion of SGD under without-replacement sampling appears in \citet{valley}. 

The SDCA method is also sometimes used with a permuted ordering \cite{SDCA}, our experiments in Section \ref{sec:experiments} show that this sometimes results in a large speedup over uniform random sampling, although it does not appear to be as reliable as with Finito.

\section{Proximal variant}
\label{sec:proximal}
We now consider composite problems of the form
\[
f(w)=\frac{1}{n}\sum_{i}f_{i}(w)+\lambda r(w),
\]
where $r$ is convex but not necessarily smooth or strongly convex. Such problems are often addressed using proximal algorithms, particularly when the proximal operator for $r$:
\[
\text{prox}_{\lambda}^{r}(z) = \text{argmin}_{x}\quad\frac{1}{2}\left\Vert x-z\right\Vert ^{2}+\lambda r(x)
\]
has a closed form solution. An example would be the use of L1 regularization. We now describe the Finito update for this setting.
First notice that when we set $w$ in the Finito method, it can be interpreted as minimizing the quantity:
\begin{align*}
B(x) &= \frac{1}{n}\sum_{i}f_{i}(\phi_{i})+\frac{1}{n}\sum_{i}\left\langle f_{i}^{\prime}(\phi_{i}),x-\phi_{i}\right\rangle\\
&+\frac{\alpha s}{2n}\sum_{i}\left\Vert x-\phi_{i}\right\Vert ^{2},
\end{align*}
with respect to $x$, for fixed $\phi_{i}$. This is related to the upper bound minimized by MISO, where $\alpha s$ is instead $L$. It is straight forward to modify this for the composite case:
\begin{align*}
B_{\lambda r}(x) &= \lambda r(x)+\frac{1}{n}\sum_{i}f_{i}(\phi_{i})+\frac{1}{n}\sum_{i}\left\langle f_{i}^{\prime}(\phi_{i}),x-\phi_{i}\right\rangle\\
 &+\frac{\alpha s}{2n}\sum_{i}\left\Vert x-\phi_{i}\right\Vert ^{2}.
\end{align*}
The minimizer of the modified $B_{\lambda r}$ can be expressed using the proximal operator as:
\[
w=\text{prox}_{\lambda/\alpha s}^{r}\left(\bar{\phi}-\frac{1}{\alpha sn}\sum_{i}f_{i}^{\prime}(\phi_{i})\right).
\]
This strongly resembles the update in the standard gradient descent setting, which for a step size of $1/L$ is
\[
w=\text{prox}_{\lambda/L}^{r}\left(w^{(k-1)}-\frac{1}{L}f^{\prime}(w^{(k-1)})\right).
\]
We have not yet developed any theory supporting the proximal variant of Finito, although empirical evidence suggests it has the same convergence rate as in the non-proximal case.

\section{Convergence proof}
\label{sec:proof}
We start by stating two simple lemmas. All expectations in the following are over the choice of index $j$ at step $k$. Quantities without superscripts are at their values at iteration $k$.
\begin{lem}
The expected step is
\[
E[w^{(k+1)}]-w=-\frac{1}{\alpha sn}f^{\prime}(w).
\]
I.e. the $w$ step is a gradient descent step in expectation ($\frac{1}{\alpha sn} \propto \frac{1}{L}$). A similar equality also holds for SGD, but not for SAG.
\end{lem}
\begin{proof}
\begin{align*}
E[w&^{(k+1)}]-w \\
& = E \left[\frac{1}{n}(w-\phi_{j})-\frac{1}{\alpha sn}\left(f_{j}^{\prime}(w)-f_{j}^{\prime}(\phi_{j})\right)\right]\\
 & = \frac{1}{n}(w-\bar{\phi})-\frac{1}{\alpha sn}f^{\prime}(w)+\frac{1}{\alpha sn^{2}}\sum_{i}f_{i}^{\prime}(\phi_{i})
\end{align*}
Now simplify $\frac{1}{n}(w-\bar{\phi})$ as $-\frac{1}{\alpha sn^{2}}\sum_{i}f_{i}^{\prime}(\phi_{i})$,
so the only term that remains is $-\frac{1}{\alpha sn}f^{\prime}(w)$.\end{proof}

\begin{lem}
\label{lem:Decomposition-of-variance}(Decomposition of variance)
We can decompose $\frac{1}{n}\sum_{i}\left\Vert w-\phi_{i}\right\Vert ^{2}$
as
\[
\frac{1}{n}\sum_{i}\left\Vert w-\phi_{i}\right\Vert ^{2}=\left\Vert w-\bar{\phi}\right\Vert ^{2}+\frac{1}{n}\sum_{i}\left\Vert \bar{\phi}-\phi_{i}\right\Vert ^{2}.
\]
\end{lem}
\begin{proof}
\begin{align*}
&\frac{1}{n}\sum_{i}\left\Vert w-\phi_{i}\right\Vert ^{2}\\
& = \left\Vert w-\bar{\phi}\right\Vert ^{2}+\frac{1}{n}\sum_{i}\left\Vert \bar{\phi}-\phi_{i}\right\Vert ^{2}+\frac{2}{n}\sum_{i}\left\langle w-\bar{\phi},\bar{\phi}-\phi_{i}\right\rangle \nonumber \\
 & = \left\Vert w-\bar{\phi}\right\Vert ^{2}+\frac{1}{n}\sum_{i}\left\Vert \bar{\phi}-\phi_{i}\right\Vert ^{2}+2\left\langle w-\bar{\phi},\bar{\phi}- \bar{\phi} \right\rangle \nonumber \\
 & = \left\Vert w-\bar{\phi}\right\Vert ^{2}+\frac{1}{n}\sum_{i}\left\Vert \bar{\phi}-\phi_{i}\right\Vert ^{2}.\label{eq:main-cancelation}
\end{align*}
\end{proof}

\subsection*{Main proof}
Our proof proceeds by construction of a Lyapunov function $T$; that is, a function that bounds a quantity of interest, and that decreases each iteration in expectation. Our Lyapunov function $T=T_1+T_2+T_3+T_4$ is composed of the sum of the following terms,
\[
T_{1}=f(\bar{\phi}),
\]
\[
T_{2}=-\frac{1}{n}\sum_{i}f_{i}(\phi_{i})-\frac{1}{n}\sum_{i}\left\langle f_{i}^{\prime}(\phi_{i}),w-\phi_{i}\right\rangle ,
\]
\[
T_{3}=-\frac{s}{2n}\sum_{i}\left\Vert w-\phi_{i}\right\Vert ^{2},
\]
\[
T_{4}=\frac{s}{2n}\sum_{i}\left\Vert \bar{\phi}-\phi_{i}\right\Vert ^{2}.
\]
We now state how each term changes between steps $k+1$ and $k$. Proofs are found in the appendix in the supplementary material:
\[
E[T_{1}^{(k+1)}]-T_{1}\leq\frac{1}{n}\left\langle f^{\prime}(\bar{\phi}),w-\bar{\phi}\right\rangle +\frac{L}{2 n^{3}}\sum_{i}\left\Vert w-\phi_{i}\right\Vert ^{2},
\]
\begin{align*}
E[T_{2}^{(k+1)}]-T_{2} & \leq -\frac{1}{n}T_{2} -\frac{1}{n}f(w)\\
 & + (\frac{1}{\alpha}-\frac{\beta}{n})\frac{1}{sn^{3}}\sum_{i}\left\Vert f_{i}^{\prime}(w)-f_{i}^{\prime}(\phi_{i})\right\Vert ^{2}\\
 & + \frac{1}{n}\left\langle \bar{\phi}-w,f^{\prime}(w)\right\rangle \\
 & -\frac{1}{n^{3}}\sum_{i}\left\langle f_{i}^{\prime}(w)-f_{i}^{\prime}(\phi_{i}),w-\phi_{i}\right\rangle ,
\end{align*}
\begin{align*}
E[T_{3}^{(k+1)}]-T_{3} & = -(\frac{1}{n}+\frac{1}{n^2})T_{3}
  + \frac{1}{\alpha n}\left\langle f^{\prime}(w),w-\bar{\phi}\right\rangle\\
 & -\frac{1}{2\alpha^{2}sn^{3}}\sum_{i}\left\Vert f_{i}^{\prime}(\phi_{i})-f_{i}^{\prime}(w)\right\Vert ^{2},
\end{align*}
\begin{align*}
E[T_{4}^{(k+1)}]-T_{4}&=-\frac{s}{2n^{2}}\sum_{i}\left\Vert \bar{\phi}-\phi_{i}\right\Vert ^{2}
+\frac{s}{2n}\left\Vert \bar{\phi}-w\right\Vert ^{2}\\
&-\frac{s}{2n^{3}}\sum_{i}\left\Vert w-\phi_{i}\right\Vert ^{2}.
\end{align*}

\begin{thm}
\label{thm:main}
Between steps $k$ and $k+1$, if $\frac{2}{\alpha}-\frac{1}{\alpha^{2}}-\beta+\frac{\beta}{\alpha}\leq 0$, $\alpha\geq2$ and $\beta \geq 2$ then
\[
E[T^{(k+1)}]-T\leq-\frac{1}{\alpha n}T.
\]
\end{thm}
\begin{proof}
We take the three lemmas above and group like terms to get
\begin{align*}
E[T^{(k+1)}]-T & \leq \frac{1}{n}\left\langle f^{\prime}(\bar{\phi}),w-\bar{\phi}\right\rangle + \frac{1}{n^{2}}\sum_{i}f_{i}(\phi_{i}) \\
 & -\frac{1}{n}f(w)+\frac{1}{n^{2}}\sum_{i}\left\langle f_{i}^{\prime}(\phi_{i}),w-\phi_{i}\right\rangle \\
 & + (1-\frac{1}{\alpha})\frac{1}{n}\left\langle f^{\prime}(w),\bar{\phi}-w\right\rangle\\
 & + (\frac{L}{s n}+1)\frac{s}{2n^{2}}\sum_{i}\left\Vert w-\phi_{i}\right\Vert ^{2}\\
 & -\frac{1}{n^{3}}\sum_{i}\left\langle f_{i}^{\prime}(w)-f_{i}^{\prime}(\phi_{i}),w-\phi_{i}\right\rangle \\
 & + (1-\frac{1}{2\alpha})\frac{1}{\alpha sn^{3}}\sum_{i}\left\Vert f_{i}^{\prime}(\phi_{i})-f_{i}^{\prime}(w)\right\Vert ^{2}\\
 & + \frac{s}{2n}\left\Vert w-\bar{\phi}\right\Vert ^{2}-\frac{s}{2n^{2}}\sum_{i}\left\Vert \bar{\phi}-\phi_{i}\right\Vert ^{2}.
\end{align*}
Next we cancel part of the first line using
\[
\frac{1}{\alpha n}\left\langle f^{\prime}(\bar{\phi}),w-\bar{\phi}\right\rangle 
\leq \frac{1}{\alpha n}f(w)- \frac{1}{\alpha n}f(\bar{\phi}) 
-\frac{s}{2\alpha n}\left\Vert w-\bar{\phi}\right\Vert ^{2},
\]
based on B3 in the Appendix. We then pull terms occurring in $-\frac{1}{\alpha n}T$ together, giving $E[T^{(k+1)}]-T \leq$
\begin{align*}
 & -\frac{1}{\alpha n}T+(1-\frac{1}{\alpha})\frac{1}{n}\left\langle f^{\prime}(\bar{\phi})-f^{\prime}(w),w-\bar{\phi}\right\rangle\\
 & + (1-\frac{1}{\alpha})\left[ -\frac{1}{n}f(w) -\frac{1}{n}T_{2} \right]\\
 & + (\frac{L}{s n}+1-\frac{1}{\alpha})\frac{s}{2n^{2}}\sum_{i}\left\Vert w-\phi_{i}\right\Vert ^{2}\\
 & -\frac{1}{n^{3}}\sum_{i}\left\langle f_{i}^{\prime}(w)-f_{i}^{\prime}(\phi_{i}),w-\phi_{i}\right\rangle \\
 & + (1-\frac{1}{2\alpha})\frac{1}{\alpha sn^{3}}\sum_{i}\left\Vert f_{i}^{\prime}(\phi_{i})-f_{i}^{\prime}(w)\right\Vert ^{2}\\
 & + (1-\frac{1}{\alpha})\frac{s}{2n}\left\Vert w-\bar{\phi}\right\Vert ^{2}-(1-\frac{1}{\alpha})\frac{s}{2n^{2}}\sum_{i}\left\Vert \bar{\phi}-\phi_{i}\right\Vert ^{2}.
\end{align*}
Next we use the standard inequality (B5)
\[
(1-\frac{1}{\alpha})\frac{1}{n}\left\langle f^{\prime}(\bar{\phi})-f^{\prime}(w),w-\bar{\phi}\right\rangle \leq-(1-\frac{1}{\alpha})\frac{s}{n}\left\Vert w-\bar{\phi}\right\Vert ^{2},
\]
which changes the bottom row to $-(1-\frac{1}{\alpha})\frac{s}{2n}\left\Vert w-\bar{\phi}\right\Vert ^{2}-(1-\frac{1}{\alpha})\frac{s}{2n^{2}}\sum_{i}\left\Vert \bar{\phi}-\phi_{i}\right\Vert ^{2}$. These two terms can then be grouped using Lemma \ref{lem:Decomposition-of-variance}, to give
\begin{align*}
E[T^{(k+1)}]-T & \leq -\frac{1}{\alpha n}T + \frac{L}{2 n^{3}}\sum_{i}\left\Vert w-\phi_{i}\right\Vert ^{2}\\
 & + (1-\frac{1}{\alpha})\left[ -\frac{1}{n}f(w) -\frac{1}{n}T_{2} \right]\\
 & -\frac{1}{n^{3}}\sum_{i}\left\langle f_{i}^{\prime}(w)-f_{i}^{\prime}(\phi_{i}),w-\phi_{i}\right\rangle \\
 & + (1-\frac{1}{2\alpha})\frac{1}{\alpha sn^{3}}\sum_{i}\left\Vert f_{i}^{\prime}(\phi_{i})-f_{i}^{\prime}(w)\right\Vert ^{2}.
\end{align*}
We use the following inequality (Corollary 11 in Appendix) to cancel against the $\sum_{i}\left\Vert w-\phi_{i}\right\Vert ^{2}$
term:
\begin{align*}
\frac{1}{\beta} \left[ -\frac{1}{n}f(w) -\frac{1}{n}T_{2} \right] 
 & \leq
 \frac{1}{n^{3}}\sum_{i}\left\langle f_{i}^{\prime}(w)-f_{i}^{\prime}(\phi_{i}),w-\phi_{i}\right\rangle \\
& - \frac{L}{2 n^{3}} \sum_{i}\left\Vert w-\phi_{i}\right\Vert ^{2} \\
&-\frac{1}{2sn^{3}}\sum_{i}\left\Vert f_{i}^{\prime}(w)-f_{i}^{\prime}(\phi_{i})\right\Vert ^{2},
\end{align*}
and then apply the following similar inequality (B7 in Appendix) to partially cancel
$\sum_{i}\left\Vert f_{i}(\phi_{i})-f_{i}(w)\right\Vert ^{2}$: 
\begin{align*}
&\left(1-\frac{1}{\alpha}-\frac{1}{\beta}\right)
\left[ -\frac{1}{n}f(w) -\frac{1}{n}T_{2} \right]\\
&\leq-\left(1-\frac{1}{\alpha}-\frac{1}{\beta}\right)\frac{\beta}{2sn^{3}}\sum_{i}\left\Vert f_{i}^{\prime}(\phi_{i})-f_{i}^{\prime}(w)\right\Vert ^{2}.
\end{align*}
Leaving us with
\begin{align*}
& E[T^{(k+1)}]-T \leq -\frac{1}{\alpha n}T\\
 & + (\frac{2}{\alpha}-\frac{1}{\alpha^{2}}-\beta+\frac{\beta}{\alpha})\frac{1}{2sn^{3}}\sum_{i}\left\Vert f_{i}^{\prime}(\phi_{i})-f_{i}^{\prime}(w)\right\Vert ^{2}.
\end{align*}
The remaining gradient norm term is non-positive under the conditions specified in our assumptions.
\end{proof}

\begin{thm}
\label{thm:lyp-to-func}The Lyapunov function bounds $f(\bar{\phi})-f(w^{*})$
as follows:
\[
f(\bar{\phi}^{(k)})-f(w^{*})\leq\alpha T^{(k)}.
\]
\end{thm}
\begin{proof}
Consider the following function, which we will call $R(x)$:
\begin{align*}
R(x)&=\frac{1}{n}\sum_{i}f_{i}(\phi_{i})+\frac{1}{n}\sum_{i}\left\langle f_{i}^{\prime}(\phi_{i}),x-\phi_{i}\right\rangle\\
& +\frac{s}{2n}\sum_{i}\left\Vert x-\phi_{i}\right\Vert ^{2}.
\end{align*}
When evaluated at its minimum with respect to $x$, which we denote $w^{\prime}=\bar{\phi}-\frac{1}{sn}\sum_{i}f_{i}^{\prime}(\phi_{i})$,
it is a lower bound on $f(w^{*})$ by strong convexity. However, we are evaluating at
$w=\bar{\phi}-\frac{1}{\alpha sn}\sum_{i}f_{i}^{\prime}(\phi_{i})$
instead in the (negated) Lyapunv function. $R$ is convex with respect to $x$, so by definition

\begin{align*}
R(w) &= R\left(\left(1-\frac{1}{\alpha}\right)\bar{\phi}+\frac{1}{\alpha}w^{\prime}\right)\\
&\leq \left(1-\frac{1}{\alpha}\right)R(\bar{\phi})+\frac{1}{\alpha}R(w^{\prime}).
\end{align*}
Therefore by the lower bounding property
\begin{align*}
f(\bar{\phi})-R(w) & \geq f(\bar{\phi})-\left(1-\frac{1}{\alpha}\right)R(\bar{\phi})-\frac{1}{\alpha}R(w^{\prime})\\
 & \geq f(\bar{\phi})-\left(1-\frac{1}{\alpha}\right)f(\bar{\phi})-\frac{1}{\alpha}f(w^{*})\\
 & = \frac{1}{\alpha}\left(f(\bar{\phi})-f(w^{*})\right).
\end{align*}
Now note that $T\geq f(\bar{\phi})-R(w)$. So
\[
f(\bar{\phi})-f(w^{*})\leq\alpha T.
\]
\end{proof} 
\begin{thm}
If the Finito method is initialized with all $\phi_{i}^{(0)}$ the same,and the assumptions of Theorem \ref{thm:main} hold, then the convergence rate is:
\[
E\left[f(\bar{\phi}^{(k)})\right]-f(w^{*})
\leq 
\frac{c}{s}\left(1-\frac{1}{\alpha n}\right)^{k}\left\Vert f^{\prime}(\bar{\phi}^{(0)})\right\Vert ^{2},
\]
with $c=\left(1-\frac{1}{2\alpha}\right)$.
\end{thm}
\begin{proof}
By unrolling Theorem \ref{thm:main}, we get
\[
E[T^{(k)}]\leq\left(1-\frac{1}{\alpha n}\right)^{k}T^{(0)}.
\]
Now using Theorem \ref{thm:lyp-to-func}
\[
E\left[f(\bar{\phi}^{(k)})\right]-f(w^{*})\leq\alpha\left(1-\frac{1}{\alpha n}\right)^{k}T^{(0)}.
\]
We need to control $T^{(0)}$ also. Since we are assuming that
all $\phi_{i}^{0}$ start the same, we have that
\begin{align*}
T^{(0)} & = f(\bar{\phi}^{(0)}) -\frac{1}{n}\sum_{i}f_{i}(\bar{\phi}^{(0)})\\
 &- \frac{1}{n}\sum_{i}\left\langle f_{i}^{\prime}(\bar{\phi}^{(0)}),w^{(0)}-\bar{\phi}^{(0)}\right\rangle -\frac{s}{2}\left\Vert w^{(0)}-\bar{\phi}^{(0)}\right\Vert ^{2}\\
 & = 0 -\left\langle f^{\prime}(\bar{\phi}^{(0)}),w^{(0)}-\bar{\phi}^{(0)}\right\rangle -\frac{s}{2}\left\Vert -\frac{1}{\alpha s}f^{\prime}(\bar{\phi}^{(0)})\right\Vert ^{2}\\
 & = \frac{1}{\alpha s}\left\Vert f^{\prime}(\bar{\phi}^{(0)})\right\Vert ^{2}-\frac{1}{2\alpha^{2}s}\left\Vert f^{\prime}(\bar{\phi}^{(0)})\right\Vert ^{2}\\
 & = \left(1-\frac{1}{2\alpha}\right)\frac{1}{\alpha s}\left\Vert f^{\prime}(\bar{\phi}^{(0)})\right\Vert ^{2}.
\end{align*}
\end{proof}

\section{Lower complexity bounds and exploiting problem structure}
\label{sec:complexity}
The theory for the class of smooth, strongly convex problems with Lipschitz continuous gradients
under first order optimization methods 
(known as $S^{1,1}_{s,L}$) is well developed. These results require the technical condition that the dimensionality of the input space $R^{m}$ is much larger than the number of iterations we will take. For simplicity we will assume this is the case in the following discussions. 

It is known that problems exist in $S^{1,1}_{s,L}$ for
 which the iterate convergence rate is bounded by:
\[
\left\Vert w^{(k)}-w^{*}\right\Vert ^{2}\geq\left(\frac{\sqrt{L/s}-1}{\sqrt{L/s}+1}\right)^{2k}\left\Vert w^{(0)}-w^{*}\right\Vert ^{2}.
\]
In fact, when $s$ and $L$ are known in advance, this rate is achieved up to a small constant factor by several methods, most notably by Nesterov's accelerated gradient descent method (\citealt{nes-opt}, \citealt{nes-book}). In order to achieve
convergence rates faster than this, additional assumptions must be made on the class of functions considered. 

Recent advances have shown that all that is required to achieve significantly faster rates is a finite sum structure, such as
in our problem setup. When the big data condition holds our method achieves a rate ~0.6065 per epoch in expectation. This rate only depends on the condition number indirectly, through the big data condition. For example, with $L/s=1,000,000$, the fastest possible rate for a black box method is a $0.996$, whereas Finito achieves a rate of $0.6065$ in expectation for $n\geq 4,000,000$, or 124x faster. The required amount of data is not unusual in modern
machine learning problems. In practice, when quasi-newton methods are used instead of accelerated methods, a speedup of 10-20x is
more common.

\subsection{Oracle class}
We now describe the (stochastic) oracle class ${FS}^{1,1}_{s,L,n}(R^{m})$ for which SAG and Finito most naturally fit.

\textbf{Function class:} $f(w)=\frac{1}{n}\sum^{n}_{i=1}f_{i}(w)$, with $f_i \in S^{1,1}_{s,L}(R^{m})$.

\textbf{Oracle:} Each query takes a point $x \in R^{m}$, and returns $j$, $f_{j}(w)$ and $f_{j}^{\prime}(w)$, with $j$ chosen uniformly at random.

\textbf{Accuracy: } Find $w$ such that $E[ \left\Vert w^{(k)}-w^{*}\right\Vert ^{2}] \leq \epsilon $.

The main choice made in formulating this definition is putting the random choice in the oracle. This restricts the methods allowed quite strongly. The alternative case, where the index $j$ is input to the oracle in addition to $x$, is also interesting. Assuming that the method has access to a source of true random indices, we call that class 
${DS}^{1,1}_{s,L,n}(R^{m})$. In Section \ref{sec:randomness} we discuss empirical evidence that suggests that faster rates are possible in ${DS}^{1,1}_{s,L,n}(R^{m})$ than for ${FS}^{1,1}_{s,L,n}(R^{m})$.

It should first be noted that there is a trivial lower bound rate for $f \in {SS}^{1,1}_{s,L,\beta}(R^{m})$ of $\left(1-\frac{1}{n}\right)$ reduction per step. Its not clear if this can be achieved for any finite $\beta$. Finito is only a factor of $2$ off this rate, namely $\left(1-\frac{1}{2n}\right)$ at $\beta=2$, and asymptotes towards this rate for very large $\beta$. SDCA, while not applicable to all problems in this class, also achieves the rate asymptotically.

Another case to consider is the smooth convex but non-strongly convex setting. We still assume Lipschitz continuous gradients. In this setting we will show that for sufficiently high dimensional input spaces, the (non-stochastic) lower complexity bound is the same for the finite sum case and cannot be better than that given by treating $f$ as a single black box function. 

The full proof is in the Appendix, but the idea is as follows: when the $f_{i}$ are not strongly convex, we can choose them such that they do not interact with each other, as long as the dimensionality is much larger than $k$. More precisely, we may choose them so that for any $x$ and $y$ and any $i \neq j$, $\langle f^{\prime}_{i}(x), f^{\prime}_{j}(y) \rangle = 0$ holds. When the functions do not interact, no optimization scheme may reduce the iterate error faster than by just handling each $f_{i}$ separately. Doing so in an in-order fashion gives the same rate as just treating $f$ using a black box method.

For strongly convex $f_{i}$, it is not possible for them to not interact in the above sense. By definition strong convexity requires a quadratic component in each $f_{i}$ that acts on all dimensions.

\begin{figure*}[t]
\begin{centering}
\begin{minipage}{234pt} \centering
\includegraphics[trim=0cm 6mm 0cm 0cm, clip=true]{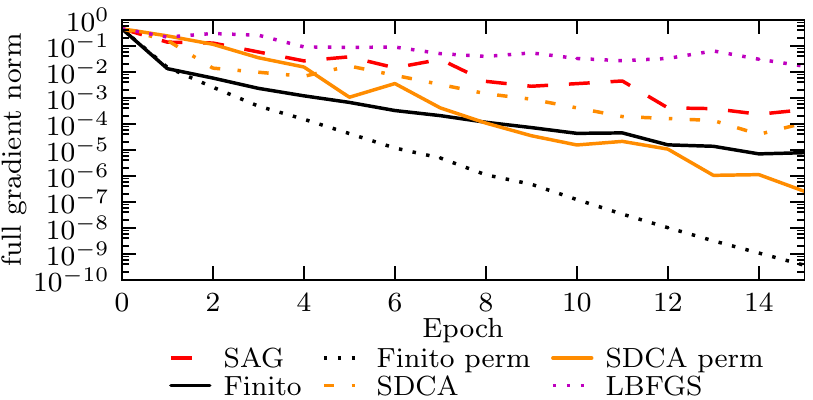}
\vskip -0.1in\caption{MNIST}
\end{minipage}
\begin{minipage}{234pt} \centering
\includegraphics[trim=0cm 6mm 0cm 0cm, clip=true]{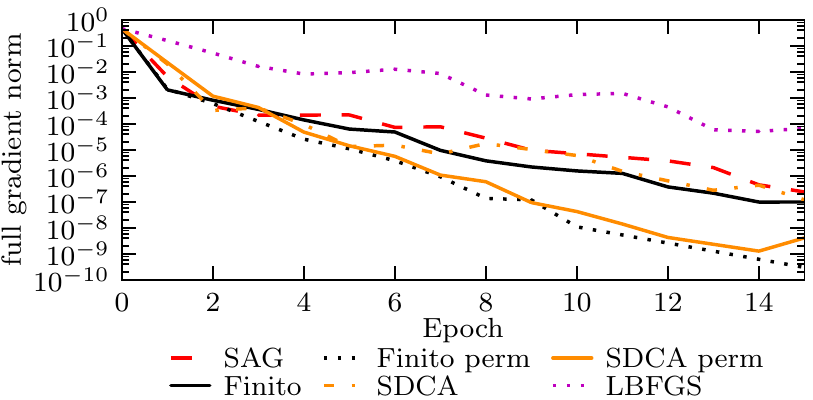} 
\vskip -0.1in\caption{ijcnn1}
\end{minipage}
\begin{minipage}{234pt} \centering
\includegraphics[trim=0cm 6mm 0cm 0cm, clip=true]{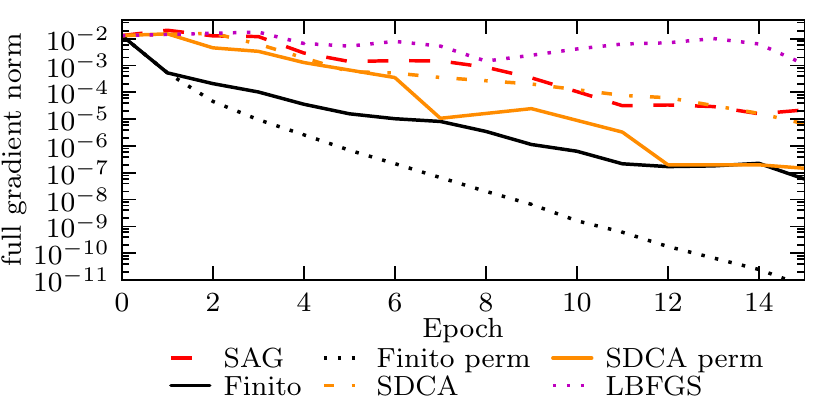}
\vskip -0.1in\caption{Covtype}
\end{minipage}
\begin{minipage}{234pt} \centering
\includegraphics[trim=0cm 6mm 0cm 0cm, clip=true]{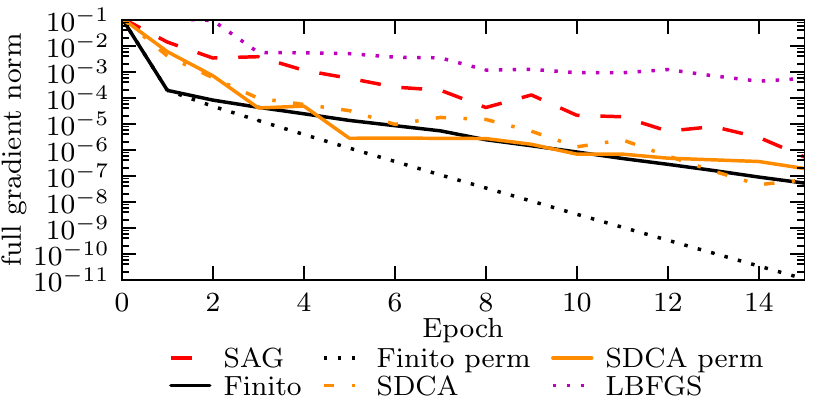}
\vskip -0.1in\caption{Million Song}
\end{minipage}
\begin{minipage}{234pt} \centering
\includegraphics[scale=1.0]{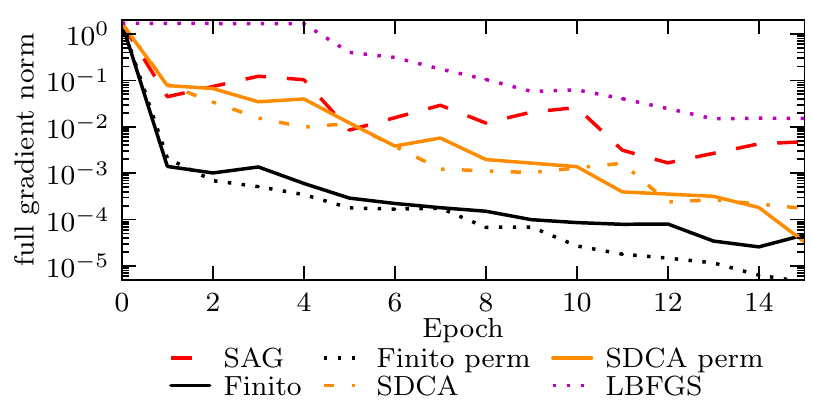}
\vskip -0.1in\caption{slice}
\end{minipage}

\end{centering}
\vskip -0.1in
\caption{Convergence rate plots for test problems}
\label{fig:conv}
\end{figure*}
\section{Experiments}
\label{sec:experiments}
In this section we compare Finito, SAG, SDCA and LBFGS. We only consider problems where the regularizer is large enough so that the big data condition holds, as this is the case our theory supports. However, in practice our method can be used with smaller step sizes in the more general case, in much the same way as SAG. 

Since we do not know the Lipschitz constant for these problems exactly, the SAG method was run for a variety of step sizes, with the one that gave the fastest rate of convergence plotted. The best step-size for SAG is usually not what the theory suggests. \citet{SAG} suggest using $\frac{1}{L}$ instead of the theoretical rate $\frac{1}{16L}$. For Finito, we find that using $\alpha=2$ is the fastest rate when the big data condition holds for any $\beta>1$. This is the step suggested by our theory when $\beta=2$. Interestingly, reducing $\alpha$ to $1$ does not improve the convergence rate. Instead we see no further improvement in our experiments.

For both SAG and Finito we used a differing step rule than suggested by the theory for the first pass. For Finito, during the first pass, since we do not have derivatives for each $\phi_{i}$ yet, we simply sum over the $k$ terms seen so far
\[
w^{(k)}=\frac{1}{k}\sum^{k}_{i} \phi_{i}^{(k)} - \frac{1}{\alpha sk}\sum^{k}_{i}f_{i}^{\prime}(\phi_{i}^{(k)}),
\]
where we process data points in index order for the first pass only. A similar trick is suggested by \citet{SAG} for SAG.

Since SDCA only applies to linear predictors, we are restricted in possible test problems. We choose log loss for 3 binary classification datasets, and quadratic loss for 2 regression tasks. For classification, we tested on the ijcnn1 and covtype datasets\footnote{\url{http://www.csie.ntu.edu.tw/~cjlin/libsvmtools/datasets/binary.html}}, as well as MNIST\footnote{\url{http://yann.lecun.com/exdb/mnist/}} classifying 0-4 against 5-9. For regression, we choose the two datasets from the UCI repository: the million song year regression dataset, and the slice-localization dataset. The training portion of the datasets are of size $5.3\times 10^5$, $5.0\times 10^4$, $6.0\times 10^4$, $4.7\times 10^5$ and $5.3\times 10^4$ respectively.

Figure \ref{fig:conv} shows the results of our experiments. Firstly we can see that LBFGS is not competitive with any of the incremental gradient methods considered. Secondly, the non-permuted SAG, Finito and SDCA often converge at very similar rates. The observed differences are usually down to the speed of the very first pass, where SAG and Finito are using the above mentioned trick to speed their convergence. After the first pass, the slopes of the line are usually comparable. When considering the methods with permutation each pass, we see a clear advantage for Finito. Interestingly, it gives very flat lines, indicating very stable convergence.

\section{Related work}
\label{sec:related}
Traditional incremental gradient methods \cite{bertsekas-incremental} have the same form as SGD, but applied to finite sums. Essentially they are the non-online analogue of SGD. Applying SGD to strongly convex problems does not yield linear convergence, and in practice it is slower than the linear-converging methods we discuss in the remainder of this section.

Besides the methods that fall under the classical Incremental gradient moniker,
SAG and MISO \citep{miso} methods are also related. MISO method falls into the class of upper bound minimization
methods, such as EM and classical gradient descent. MISO is essentially the Finito method, but with step sizes $n$
times smaller. When using these larger step sizes, the method is no longer a upper bound minimization method. Our method can be seen as MISO, but with a step size scheme that gives neither a lower nor upper bound minimisation method. While this work was under peer review, a tech report (\citet{miso2}) was put on arXiv that establishes the convergence rate of MISO with step $\alpha=1$ and with $\beta=2$ as $1-\frac{1}{3n}$ per step. This similar but not quite as good as the $1-\frac{1}{2n}$ rate we establish.

Stochastic Dual Coordinate descent \cite{SDCA} also gives fast convergence rates on problems for which it is applicable. It requires 
computing the convex conjugate of each $f_{i}$, which makes it more complex to implement. For the best performance it has to take
advantage of the structure of the losses also. For simple linear classification and regression problems it can be effective. When using
a sparse dataset, it is a better choice than Finito due to the memory requirements. For linear predictors, its theoretical convergence rate of  $\left(1-\frac{\beta}{(1+\beta)n}\right)$ per step is a little faster than what we establish for Finito, however it does not appear to be faster in our experiments.

\section{Conclusion}
\label{sec:conclusion}
We have presented a new method for minimization of finite sums of smooth strongly convex functions, when there is a sufficiently large number of terms in the summation. We additionally develop some theory for the lower complexity bounds on this class, and show the empirical performance of our method.

\bibliography{papera}
\bibliographystyle{icml2014}

\end{document}


\title{Appendix}

\maketitle

\section{Basic convexity inequalities}

The following inequalities are classical. See Nesterov 1998 for proofs.
They hold for all $x$ \& $y$, when $f\in S_{s,L}^{1,1}$. 

$\text{(B1) }\quad f(y)\leq f(x)+\left\langle f^{\prime}(x),y-x\right\rangle +\frac{L}{2}\left\Vert x-y\right\Vert ^{2}$

$\text{(B2) \quad}f(y)\geq f(x)+\left\langle f^{\prime}(x),y-x\right\rangle +\frac{1}{2L}\left\Vert f^{\prime}(x)-f^{\prime}(y)\right\Vert ^{2}$

$\text{(B3) \quad}f(y)\geq f(x)+\left\langle f^{\prime}(x),y-x\right\rangle +\frac{s}{2}\left\Vert x-y\right\Vert ^{2}$

$\text{(B4) \quad}\left\langle f^{\prime}(x)-f^{\prime}(y),x-y\right\rangle \geq\frac{1}{L}\left\Vert f^{\prime}(x)-f^{\prime}(y)\right\Vert ^{2}$

$\text{(B5) \quad}\left\langle f^{\prime}(x)-f^{\prime}(y),x-y\right\rangle \geq s\left\Vert x-y\right\Vert ^{2}$

We also use variants of B2 and B3 that are summed over each $f_{i}$,
with $x=\phi_{i}$ and $y=w$:
\[
f(w)\geq\frac{1}{n}\sum_{i}f_{i}(\phi_{i})+\frac{1}{n}\sum_{i}\left\langle f_{i}^{\prime}(\phi_{i}),w-\phi_{i}\right\rangle +\frac{1}{2Ln}\sum_{i}\left\Vert f^{\prime}(x)-f^{\prime}(y)\right\Vert ^{2}
\]

\[
f(w)\geq\frac{1}{n}\sum_{i}f_{i}(\phi_{i})+\frac{1}{n}\sum_{i}\left\langle f_{i}^{\prime}(\phi_{i}),w-\phi_{i}\right\rangle +\frac{s}{2n}\sum_{i}\left\Vert w-\phi_{i}\right\Vert ^{2}
\]

These are used in the following negated and rearranged form: 
\begin{eqnarray*}
-f(w)-T_{2} & = & -f(w)+\frac{1}{n}\sum_{i}f_{i}(\phi_{i})+\frac{1}{n}\sum_{i}\left\langle f_{i}^{\prime}(\phi_{i}),w-\phi_{i}\right\rangle \\
\text{(B6) \quad}\therefore-f(w)-T_{2} & \leq & -\frac{s}{2n}\sum_{i}\left\Vert w-\phi_{i}\right\Vert ^{2}
\end{eqnarray*}

\begin{eqnarray*}
\text{(B7) \quad}-f(w)-T_{2} & \leq & -\frac{1}{2Ln}\sum_{i}\left\Vert f^{\prime}(w)-f^{\prime}(\phi_{i})\right\Vert ^{2}.
\end{eqnarray*}

\section{Lyapunov term bounds}

Simplifying each Lyapunov term is fairly straight forward. We use
extensively that $\phi_{j}^{(k+1)}=w$, and that $\phi_{i}^{(k+1)}=\phi_{i}$
for $i\neq j$. Note also that 
\[
\text{(B8) \quad}w^{(k+1)}-w=\frac{1}{n}(w-\phi_{j})+\frac{1}{\alpha sn}\left[f_{j}^{\prime}(\phi_{j})-f_{j}^{\prime}(w)\right].
\]

\begin{lem}
Between steps $k$ and $k+1$, the $T_{1}=f(\bar{\phi})$ term changes
as follows:
\[
E[T_{1}^{(k+1)}]-T_{1}\leq\frac{1}{n}\left\langle f^{\prime}(\bar{\phi}),w-\bar{\phi}\right\rangle +\frac{L}{2n^{3}}\sum_{i}\left\Vert w-\phi_{i}\right\Vert ^{2}.
\]
\end{lem}
\begin{proof}
First we use the standard Lipschitz upper bound (B1):
\[
f(y)\leq f(x)+\left\langle f^{\prime}(x),y-x\right\rangle +\frac{L}{2}\left\Vert x-y\right\Vert ^{2}.
\]

We can apply this using $y=\bar{\phi}^{(k+1)}=\bar{\phi}+\frac{1}{n}(w-\phi_{j})$
and $x=\bar{\phi}$:

\begin{eqnarray*}
f(\bar{\phi}^{(k+1)}) & \leq & f(\bar{\phi})+\frac{1}{n}\left\langle f^{\prime}(\bar{\phi}),w-\phi_{j}\right\rangle +\frac{L}{2n^{2}}\left\Vert w-\phi_{j}\right\Vert ^{2}.
\end{eqnarray*}

We now take expectations over $j$, giving:
\[
E[f(\bar{\phi}^{(k+1)})]-f(\bar{\phi})\leq\frac{1}{n}\left\langle f^{\prime}(\bar{\phi}),w-\bar{\phi}\right\rangle +\frac{L}{2n^{3}}\sum_{i}\left\Vert w-\phi_{i}\right\Vert ^{2}.
\]
\end{proof}
\begin{lem}
Between steps $k$ and $k+1$, the $T_{2}=-\frac{1}{n}\sum_{i}f_{i}(\phi_{i})-\frac{1}{n}\sum_{i}\left\langle f_{i}^{\prime}(\phi_{i}),w-\phi_{i}\right\rangle $
term changes as follows:
\begin{eqnarray*}
E[T_{2}^{(k+1)}]-T_{2} & \leq & -\frac{1}{n}T_{2}-\frac{1}{n}f(w)\\
 & + & (\frac{1}{\alpha}-\frac{\beta}{n})\frac{1}{sn^{3}}\sum_{i}\left\Vert f_{i}^{\prime}(w)-f_{i}^{\prime}(\phi_{i})\right\Vert ^{2}\\
 & + & \frac{1}{n}\left\langle \bar{\phi}-w,f^{\prime}(w)\right\rangle -\frac{1}{n^{3}}\sum_{i}\left\langle f_{i}^{\prime}(w)-f_{i}^{\prime}(\phi_{i}),w-\phi_{i}\right\rangle .
\end{eqnarray*}
\end{lem}
\begin{proof}
We introduce the notation $T_{21}=-\frac{1}{n}\sum_{i}f_{i}(\phi_{i})$
and $T_{22}=-\frac{1}{n}\sum_{i}\left\langle f_{i}^{\prime}(\phi_{i}),w-\phi_{i}\right\rangle $.
We simplify the change in $T_{21}$ first using $\phi_{j}^{(k+1)}=w$:
\begin{eqnarray*}
T_{21}^{(k+1)}-T_{21} & = & -\frac{1}{n}\sum_{i}f_{i}(\phi_{i}^{(k+1)})+\frac{1}{n}\sum_{i}f_{i}(\phi_{i})\\
 & = & -\frac{1}{n}\sum_{i}f_{i}(\phi_{i})+\frac{1}{n}f_{j}(\phi_{j})-\frac{1}{n}f_{j}(w)+\frac{1}{n}\sum_{i}f_{i}(\phi_{i})\\
 & = & \frac{1}{n}f_{j}(\phi_{j})-\frac{1}{n}f_{j}(w)
\end{eqnarray*}

Now we simplify the change in $T_{22}$:

\[
T_{22}^{(k+1)}-T_{22}=-\frac{1}{n}\sum_{i}\left\langle f_{i}^{\prime}(\phi_{i}^{(k+1)}),w^{(k+1)}-w+w-\phi_{i}^{(k+1)}\right\rangle -T_{22}
\]
\begin{equation}
\therefore T_{22}^{(k+1)}-T_{22}=-\frac{1}{n}\sum_{i}\left\langle f_{i}^{\prime}(\phi_{i}^{(k+1)}),w-\phi_{i}^{(k+1)}\right\rangle -T_{22}-\frac{1}{n}\sum_{i}\left\langle f_{i}^{\prime}(\phi_{i}^{(k+1)}),w^{(k+1)}-w\right\rangle .\label{eq:t2-start}
\end{equation}

We now simplying the first two terms using $\phi_{j}^{(k+1)}=w$:
\begin{eqnarray*}
-\frac{1}{n}\sum_{i}\left\langle f_{i}^{\prime}(\phi_{i}^{(k+1)}),w-\phi_{i}^{(k+1)}\right\rangle -T_{22} & = & T_{22}-\frac{1}{n}\left\langle f_{j}^{\prime}(\phi_{j}),w-\phi_{j}\right\rangle +\frac{1}{n}\left\langle f_{j}^{\prime}(w),w-w\right\rangle -T_{22}\\
 & = & \frac{1}{n}\left\langle f_{j}^{\prime}(\phi_{j}),w-\phi_{j}\right\rangle .
\end{eqnarray*}
The last term of Equation \ref{eq:t2-start} expands further: 
\begin{eqnarray}
-\frac{1}{n}\sum_{i}\left\langle f_{i}^{\prime}(\phi_{i}^{(k+1)}),w^{(k+1)}-w\right\rangle  & = & -\frac{1}{n}\left\langle \sum_{i}f_{i}^{\prime}(\phi_{i})-f_{j}^{\prime}(\phi_{j})+f_{j}^{\prime}(w),w^{(k+1)}-w\right\rangle \nonumber \\
 & = & -\frac{1}{n}\left\langle \sum_{i}f_{i}^{\prime}(\phi_{i}),w^{(k+1)}-w\right\rangle -\frac{1}{n}\left\langle f_{j}^{\prime}(w)-f_{j}^{\prime}(\phi_{j}),w^{(k+1)}-w\right\rangle .\label{eq:t2-ip-exp}
\end{eqnarray}

The second inner product term in \ref{eq:t2-ip-exp} simplifies further
using B8:
\begin{eqnarray*}
-\frac{1}{n}\left\langle f_{j}^{\prime}(w)-f_{j}^{\prime}(\phi_{j}),w^{(k+1)}-w\right\rangle  & = & -\frac{1}{n}\left\langle f_{j}^{\prime}(w)-f_{j}^{\prime}(\phi_{j}),\frac{1}{n}(w-\phi_{j})+\frac{1}{\alpha sn}\left[f_{j}^{\prime}(\phi_{j})-f_{j}^{\prime}(w)\right]\right\rangle \\
 & = & -\frac{1}{n^{2}}\left\langle f_{j}^{\prime}(w)-f_{j}^{\prime}(\phi_{j}),w-\phi_{j}\right\rangle -\frac{1}{\alpha sn^{2}}\left\langle f_{j}^{\prime}(w)-f_{j}^{\prime}(\phi_{j}),f_{j}^{\prime}(\phi_{j})-f_{j}^{\prime}(w)\right\rangle .
\end{eqnarray*}

We simplify the second term: 
\begin{eqnarray*}
-\frac{1}{\alpha sn^{2}}\left\langle f_{j}^{\prime}(w)-f_{j}^{\prime}(\phi_{j}),f_{j}^{\prime}(\phi_{j})-f_{j}^{\prime}(w)\right\rangle  & = & \frac{1}{\alpha sn^{2}}\left\Vert f_{j}^{\prime}(w)-f_{j}^{\prime}(\phi_{j})\right\Vert ^{2}.
\end{eqnarray*}

Grouping all remaining terms gives:
\begin{eqnarray*}
T_{2}^{(k+1)}-T_{2} & \leq & \frac{1}{n}f_{j}(\phi_{j})+\frac{1}{n}\left\langle f_{j}^{\prime}(\phi_{j}),w-\phi_{j}\right\rangle -\frac{1}{n}f_{j}(w)\\
 & + & \frac{1}{\alpha sn^{2}}\left\Vert f_{j}^{\prime}(w)-f_{j}^{\prime}(\phi_{j})\right\Vert ^{2}-\frac{1}{n^{2}}\left\langle f_{j}^{\prime}(w)-f_{j}^{\prime}(\phi_{j}),w-\phi_{j}\right\rangle \\
 & - & \frac{1}{n}\left\langle \sum_{i}f_{i}^{\prime}(\phi_{i}),w^{(k+1)}-w\right\rangle .
\end{eqnarray*}

We now take expectations of each remaining term. For the bottom inner
product we use Lemma 1:
\begin{eqnarray*}
-\frac{1}{n}\left\langle \sum_{i}f_{i}^{\prime}(\phi_{i}),w^{(k+1)}-w\right\rangle  & = & \frac{1}{\alpha sn^{2}}\left\langle \sum_{i}f_{i}^{\prime}(\phi_{i}),f^{\prime}(w)\right\rangle \\
 & = & \frac{1}{n}\left\langle \bar{\phi}-w,f^{\prime}(w)\right\rangle .
\end{eqnarray*}

Taking expectations of the remaining terms is straight forward. We
get:
\begin{eqnarray*}
E[T_{2}^{(k+1)}]-T_{2} & \leq & \frac{1}{n^{2}}\sum_{i}f_{i}(\phi_{i})-\frac{1}{n}f(w)+\frac{1}{n^{2}}\sum_{i}\left\langle f_{i}^{\prime}(\phi_{i}),w-\phi_{i}\right\rangle \\
 & + & \frac{1}{\alpha sn^{3}}\sum_{i}\left\Vert f_{i}^{\prime}(w)-f_{i}^{\prime}(\phi_{i})\right\Vert ^{2}-\frac{1}{n^{3}}\sum_{i}\left\langle f_{i}^{\prime}(w)-f_{i}^{\prime}(\phi_{i}),w-\phi_{i}\right\rangle \\
 & + & \frac{1}{n}\left\langle \bar{\phi}-w,f^{\prime}(w)\right\rangle .
\end{eqnarray*}
\end{proof}
\begin{lem}
Between steps $k$ and $k+1$, the $T_{3}=-\frac{s}{2n}\sum_{i}\left\Vert w-\phi_{i}\right\Vert ^{2}$
term changes as follows:
\begin{eqnarray*}
E[T_{3}^{(k+1)}]-T_{3} & = & -(1+\frac{1}{n})\frac{1}{n}T_{3}\\
 & + & \frac{1}{\alpha n}\left\langle f^{\prime}(w),w-\bar{\phi}\right\rangle -\frac{1}{2\alpha^{2}sn^{3}}\sum_{i}\left\Vert f_{i}^{\prime}(\phi_{i})-f_{i}^{\prime}(w)\right\Vert ^{2}.
\end{eqnarray*}
\end{lem}
\begin{proof}
We expand as:
\begin{eqnarray}
T_{3}^{(k+1)} & = & -\frac{s}{2n}\sum_{i}\left\Vert w^{(k+1)}-\phi_{i}^{(k+1)}\right\Vert ^{2}\nonumber \\
 & = & -\frac{s}{2n}\sum_{i}\left\Vert w^{(k+1)}-w+w-\phi_{i}^{(k+1)}\right\Vert ^{2}\\
 & = & -\frac{s}{2}\left\Vert w^{(k+1)}-w\right\Vert ^{2}-\frac{s}{2n}\sum_{i}\left\Vert w-\phi_{i}^{(k+1)}\right\Vert ^{2}-\frac{s}{n}\sum_{i}\left\langle w^{(k+1)}-w,w-\phi_{i}^{(k+1)}\right\rangle .\label{eq:t3-initial}
\end{eqnarray}

We expand the three terms on the right separately. For the first term:
\begin{eqnarray}
-\frac{s}{2}\left\Vert w^{(k+1)}-w\right\Vert ^{2} & = & -\frac{s}{2}\left\Vert \frac{1}{n}(w-\phi_{j})+\frac{1}{\alpha sn}\left(f_{j}(\phi_{j})-f_{j}(w)\right)\right\Vert ^{2}\nonumber \\
 & =- & \frac{s}{2n^{2}}\left\Vert w-\phi_{j}\right\Vert ^{2}-\frac{1}{2\alpha^{2}sn^{2}}\left\Vert f_{j}(\phi_{j})-f_{j}(w)\right\Vert ^{2}\nonumber \\
 & - & \frac{1}{\alpha n^{2}}\left\langle f_{j}(\phi_{j})-f_{j}(w),w-\phi_{j}\right\rangle .\label{eq:wkp-minus-w-expansion}
\end{eqnarray}

For the second term of Equation \ref{eq:t3-initial}, using $\phi_{j}^{(k+1)}=w$:
\begin{eqnarray*}
-\frac{s}{2n}\sum_{i}\left\Vert w-\phi_{i}^{(k+1)}\right\Vert ^{2} & = & -\frac{s}{2n}\sum_{i}\left\Vert w-\phi_{i}\right\Vert ^{2}+\frac{s}{2n}\left\Vert w-\phi_{j}\right\Vert ^{2}\\
 & = & T_{3}+\frac{s}{2n}\left\Vert w-\phi_{j}\right\Vert ^{2}.
\end{eqnarray*}

For the third term of Equation \ref{eq:t3-initial}:
\begin{eqnarray}
-\frac{s}{n}\sum_{i}\left\langle w^{(k+1)}-w,w-\phi_{i}^{(k+1)}\right\rangle  & = & -\frac{s}{n}\sum_{i}\left\langle w^{(k+1)}-w,w-\phi_{i}\right\rangle +\frac{s}{n}\left\langle w^{(k+1)}-w,w-\phi_{j}\right\rangle \nonumber \\
 & = & -s\left\langle w^{(k+1)}-w,w-\frac{1}{n}\sum_{i}\phi_{i}\right\rangle +\frac{s}{n}\left\langle w^{(k+1)}-w,w-\phi_{j}\right\rangle .\label{eq:ip-terms-t3}
\end{eqnarray}

The second inner product term in Equation \ref{eq:ip-terms-t3} becomes
(using B8):
\begin{eqnarray*}
\frac{s}{n}\left\langle w^{(k+1)}-w,w-\phi_{j}\right\rangle  & = & \frac{s}{n}\left\langle \frac{1}{n}(w-\phi_{j})+\frac{1}{\alpha sn}\left[f_{j}^{\prime}(\phi_{j})-f_{j}^{\prime}(w)\right],w-\phi_{j}\right\rangle \\
 & = & \frac{s}{n^{2}}\left\Vert w-\phi_{j}\right\Vert ^{2}+\frac{1}{\alpha n^{2}}\left\langle f_{j}^{\prime}(\phi_{j})-f_{j}^{\prime}(w),w-\phi_{j}\right\rangle .
\end{eqnarray*}

Notice that the inner product term here cancels with the one in \ref{eq:wkp-minus-w-expansion}.

Now we can take expectations of each remaining term. Recall that $E[w^{(k+1)}]-w=-\frac{1}{\alpha sn}f^{\prime}(w)$,
so the first inner product term in \ref{eq:ip-terms-t3} becomes:
\[
-sE\left[\left\langle w^{(k+1)}-w,w-\frac{1}{n}\sum_{i}\phi_{i}\right\rangle \right]=\frac{1}{\alpha n}\left\langle f^{\prime}(w),w-\bar{\phi}\right\rangle .
\]

All other terms don't simplify under expectations. So the result is:
\begin{eqnarray*}
E[T_{3}^{(k+1)}]-T_{3} & = & (\frac{1}{2}-\frac{1}{n})\frac{s}{n^{2}}\sum_{i}\left\Vert w-\phi_{i}\right\Vert ^{2}\\
 & + & \frac{1}{\alpha n}\left\langle f^{\prime}(w),w-\bar{\phi}\right\rangle -\frac{1}{2\alpha^{2}sn^{3}}\sum_{i}\left\Vert f_{i}(\phi_{i})-f_{i}(w)\right\Vert ^{2}.
\end{eqnarray*}
\end{proof}
\begin{lem}
Between steps $k$ and $k+1$, the $T_{4}=\frac{s}{2n}\sum_{i}\left\Vert \bar{\phi}-\phi_{i}\right\Vert ^{2}$
term changes as follows:

\[
E[T_{4}^{(k+1)}]-T_{4}=-\frac{s}{2n^{2}}\sum_{i}\left\Vert \bar{\phi}-\phi_{i}\right\Vert ^{2}+\frac{s}{2n}\left\Vert \bar{\phi}-w\right\Vert ^{2}-\frac{s}{2n^{3}}\sum_{i}\left\Vert w-\phi_{i}\right\Vert ^{2}.
\]
\end{lem}
\begin{proof}
Note that $\bar{\phi}^{(k+1)}-\bar{\phi}=\frac{1}{n}(w-\phi_{j})$,
so:
\begin{eqnarray*}
T_{4}^{(k+1)} & = & \frac{s}{2n}\sum_{i}\left\Vert \bar{\phi}^{(k+1)}-\bar{\phi}+\bar{\phi}-\phi_{i}^{(k+1)}\right\Vert ^{2}\\
 & = & \frac{s}{2n}\sum_{i}\left(\left\Vert \bar{\phi}^{(k+1)}-\bar{\phi}\right\Vert ^{2}+\left\Vert \bar{\phi}-\phi_{i}^{(k+1)}\right\Vert ^{2}+2\left\langle \bar{\phi}^{(k+1)}-\bar{\phi},\bar{\phi}-\phi_{i}^{(k+1)}\right\rangle \right)\\
 & = & \frac{s}{2n}\sum_{i}\left(\left\Vert \frac{1}{n}(w-\phi_{j})\right\Vert ^{2}+\left\Vert \bar{\phi}-\phi_{i}^{(k+1)}\right\Vert ^{2}+\frac{2}{n}\left\langle w-\phi_{j},\bar{\phi}-\phi_{i}^{(k+1)}\right\rangle \right).
\end{eqnarray*}
Now using $\frac{1}{n}\sum_{i}\left(\bar{\phi}-\phi_{i}^{(k+1)}\right)=\bar{\phi}-\bar{\phi}^{(k+1)}=-\frac{1}{n}(w-\phi_{j})$
to simplify the inner product term:
\begin{eqnarray}
 & = & \frac{s}{2n^{2}}\left\Vert w-\phi_{j}\right\Vert ^{2}+\frac{s}{2n}\sum_{i}\left\Vert \bar{\phi}-\phi_{i}^{(k+1)}\right\Vert ^{2}+\frac{s}{n^{2}}\left\langle w-\phi_{j},\phi_{j}-w\right\rangle \nonumber \\
 & = & \frac{s}{2n^{2}}\left\Vert w-\phi_{j}\right\Vert ^{2}+\frac{s}{2n}\sum_{i}\left\Vert \bar{\phi}-\phi_{i}^{(k+1)}\right\Vert ^{2}-\frac{s}{n}\left\Vert w-\phi_{j}\right\Vert ^{2}\nonumber \\
 & = & \frac{s}{2n}\sum_{i}\left\Vert \bar{\phi}-\phi_{i}^{(k+1)}\right\Vert ^{2}-\frac{s}{2n}\left\Vert w-\phi_{j}\right\Vert ^{2}\nonumber \\
 & = & \frac{s}{2n}\sum_{i}\left\Vert \bar{\phi}-\phi_{i}\right\Vert ^{2}-\frac{s}{2n}\left\Vert \bar{\phi}-\phi_{j}\right\Vert ^{2}+\frac{s}{2n}\left\Vert \bar{\phi}-w\right\Vert ^{2}-\frac{s}{2n^{2}}\left\Vert w-\phi_{j}\right\Vert ^{2}.\label{eq:right-term}
\end{eqnarray}

Taking expectations gives the result.\end{proof}
\begin{lem}
\label{thm:strong-lb}Let $f\in S_{s,L}$. Then we have:
\[
f(x)\geq f(y)+\left\langle f^{\prime}(y),x-y\right\rangle +\frac{1}{2\left(L-s\right)}\left\Vert f^{\prime}(x)-f^{\prime}(y)\right\Vert ^{2}+\frac{sL}{2\left(L-s\right)}\left\Vert y-x\right\Vert ^{2}+\frac{s}{\left(L-s\right)}\left\langle f^{\prime}(x)-f^{\prime}(y),y-x\right\rangle .
\]
\end{lem}
\begin{proof}
Define the function $g$ as $g(x)=f(x)-\frac{s}{2}\left\Vert x\right\Vert ^{2}$.
Then the gradient is $g^{\prime}(x)=f^{\prime}(x)-sx$. $g$ has a
lipschitz gradient with with constant $L-s$. By convexity we have:
\[
g(x)\geq g(y)+\left\langle g^{\prime}(y),x-y\right\rangle +\frac{1}{2(L-s)}\left\Vert g^{\prime}(x)-g^{\prime}(y)\right\Vert ^{2}.
\]

Now replacing $g$ with $f$:

\[
f(x)-\frac{s}{2}\left\Vert x\right\Vert ^{2}\geq f(y)-\frac{s}{2}\left\Vert y\right\Vert ^{2}+\left\langle f^{\prime}(y)-sy,x-y\right\rangle +\frac{1}{2\left(L-s\right)}\left\Vert f^{\prime}(x)-sx-f^{\prime}(y)+sy\right\Vert ^{2}.
\]

Note that

\begin{eqnarray*}
\frac{1}{2\left(L-s\right)}\left\Vert f^{\prime}(x)-sx-f^{\prime}(y)+sy\right\Vert ^{2} & = & \frac{1}{2\left(L-s\right)}\left\Vert f^{\prime}(x)-f^{\prime}(y)\right\Vert ^{2}+\frac{s^{2}}{2\left(L-s\right)}\left\Vert y-x\right\Vert ^{2}\\
 &  & \frac{s}{\left(L-s\right)}\left\langle f^{\prime}(x)-f^{\prime}(y),y-x\right\rangle ,
\end{eqnarray*}

so:

\begin{eqnarray*}
f(x) & \geq & f(y)+\left\langle f^{\prime}(y),x-y\right\rangle +\frac{1}{2\left(L-s\right)}\left\Vert f^{\prime}(x)-f^{\prime}(y)\right\Vert ^{2}+\frac{s^{2}}{2\left(L-s\right)}\left\Vert y-x\right\Vert ^{2}\\
 &  & +\frac{s}{2}\left\Vert x\right\Vert ^{2}-\frac{s}{2}\left\Vert y\right\Vert ^{2}+\frac{s}{\left(L-s\right)}\left\langle f^{\prime}(x)-f^{\prime}(y),y-x\right\rangle -s\left\langle y,x-y\right\rangle .
\end{eqnarray*}

Now using:

\[
\frac{s}{2}\left\Vert x\right\Vert ^{2}-s\left\langle y,x\right\rangle =-\frac{s}{2}\left\Vert y\right\Vert ^{2}+\frac{s}{2}\left\Vert x-y\right\Vert ^{2},
\]

we get:

\begin{eqnarray*}
f(x) & \geq & f(y)+\left\langle f^{\prime}(y),x-y\right\rangle +\frac{1}{2\left(L-s\right)}\left\Vert f^{\prime}(x)-f^{\prime}(y)\right\Vert ^{2}+\frac{s^{2}}{2\left(L-s\right)}\left\Vert x-y\right\Vert ^{2}\\
 &  & -s\left\Vert y\right\Vert ^{2}+\frac{s}{2}\left\Vert x-y\right\Vert ^{2}+\frac{s}{\left(L-s\right)}\left\langle f^{\prime}(x)-f^{\prime}(y),y-x\right\rangle +s\left\langle y,y\right\rangle 
\end{eqnarray*}

Note the norm $y$ terms cancel, and:
\begin{eqnarray*}
\frac{s}{2}\left\Vert x-y\right\Vert ^{2}+\frac{s^{2}}{2\left(L-s\right)}\left\Vert x-y\right\Vert ^{2} & = & \frac{(L-s)s+s^{2}}{2\left(L-s\right)}\left\Vert x-y\right\Vert ^{2}\\
 & = & \frac{sL}{2\left(L-s\right)}\left\Vert x-y\right\Vert ^{2}.
\end{eqnarray*}

So:

\begin{eqnarray*}
f(x) & \geq & f(y)+\left\langle f^{\prime}(y),x-y\right\rangle +\frac{1}{2\left(L-s\right)}\left\Vert f^{\prime}(x)-f^{\prime}(y)\right\Vert ^{2}+\frac{sL}{2\left(L-s\right)}\left\Vert y-x\right\Vert ^{2}\\
 &  & +\frac{s}{\left(L-s\right)}\left\langle f^{\prime}(x)-f^{\prime}(y),y-x\right\rangle 
\end{eqnarray*}
\end{proof}
\begin{cor}
\label{cor:constant-corr}Take $f(x)=\frac{1}{n}\sum_{i}f_{i}(x)$,
with the big data condition holding with constant $\beta$. Then for
any $x$ and $\phi_{i}$ vectors: 
\begin{eqnarray*}
f(x) & \geq & \frac{1}{n}\sum_{i}f_{i}(\phi_{i})+\frac{1}{n}\sum_{i}\left\langle f_{i}^{\prime}(\phi_{i}),x-\phi_{i}\right\rangle +\frac{\beta}{2sn^{2}}\sum_{i}\left\Vert f_{i}^{\prime}(x)-f_{i}^{\prime}(\phi_{i})\right\Vert ^{2}\\
 &  & +\frac{\beta L}{2n^{2}}\sum_{i}\left\Vert x-\phi_{i}\right\Vert ^{2}+\frac{\beta}{n^{2}}\sum_{i}\left\langle f_{i}^{\prime}(x)-f_{i}^{\prime}(\phi_{i}),\phi_{i}-x\right\rangle .
\end{eqnarray*}
\end{cor}
\begin{proof}
We apply Lemma \ref{thm:strong-lb} to each $f_{i}$ , but instead
of using the actual constant $L$, we use $\frac{sn}{\beta}+s$, which
under the big data assumption is larger than $L$:

\[
f_{i}(x)\geq f_{i}(\phi_{i})+\left\langle f_{i}^{\prime}(\phi_{i}),x-\phi_{i}\right\rangle +\frac{\beta}{2sn}\left\Vert f_{i}^{\prime}(x)-f_{i}^{\prime}(\phi_{i})\right\Vert ^{2}+\frac{\beta L}{2n}\left\Vert x-\phi_{i}\right\Vert ^{2}+\frac{\beta}{n}\left\langle f_{i}^{\prime}(x)-f_{i}^{\prime}(\phi_{i}),\phi_{i}-x\right\rangle .
\]

Averaging over $i$ gives the result.
\end{proof}

\section{Lower complexity bounds}

In this section we use the following technical assumption, as used
in Nesterov (1998):

\textit{Assumption 1: An optimization method at step $k$ may only
invoke the oracle with a point $x^{(k)}$ that is of the form:
\[
x^{(k)}=x^{(0)}+\sum_{i}a_{i}g^{(i)},
\]
}\emph{where $g^{(i)}$ is the derivative returned by the oracle at
step $i$, and $a_{i}\in R$. }

This assumption prevents an optimization method from just guessing
the correct solution without doing any work. Virtually all optimization
methods fall into under this assumption.

\subsection*{Simple $(1-\frac{1}{n})^{k}$ bound}

Any procedure that minimizes a sum of the form $f(w)=\frac{1}{n}\sum_{i}f_{i}(w)$
by uniform random access of $f_{i}$ is restricted by the requirement
that it has to actually see each term at least once in order to find
the minimum. This leads to a $\left(1-\frac{1}{n}\right)^{k}$ rate
in expectation. We now formalize such an argument. We will work in
$R^{n}$, matching the dimensionality of the problem to the number
of terms in the summation.
\begin{thm}
For any $f\in{FS}_{1,n,n}^{1,1}(R^{n})$, we have that a $k$ step
optimization procedure gives:
\[
E[f(w)]-f(w^{*})\geq\left(1-\frac{1}{n}\right)^{k}\left(f(w^{(0)})-f(w^{*}\right))
\]
\end{thm}
\begin{proof}
We will exhibit a simple worst-case problem. Without loss of generality
we assume that the first oracle access by the optimization procedure
is at $w=0$. In any other case, we shift our space in the following
argument approprately. 

Let $f(w)=\frac{1}{n}\sum_{i}\left[\frac{n}{2}\left(w_{i}-1\right)^{2}+\frac{1}{2}\left\Vert w\right\Vert ^{2}\right]$.
Then clearly the solution is $w_{i}=\frac{1}{2}$ for each $i$, with
minimum of $f(w^{*})=\frac{n}{4}$. For $w=0$ we have $f(0)=\frac{n}{2}$.
Since the derivative of each $f_{j}$ is $0$ on the $i$th component
if we have not yet seen $f_{i}$, the value of each $w_{i}$ remains
$0$ unless term $i$ has been seen. 

Let $v^{(k)}$ be the number of unique terms we have not seen up to
step $k$. Between steps $k$ and $k+1$, $v$ decreases by $1$ with
probably $\frac{v}{n}$ and stays the same otherwise. So 
\[
E[v^{(k+1)}|v^{(k)}]=v^{(k)}-\frac{v^{(k)}}{n}=\left(1-\frac{1}{n}\right)v^{(k)}.
\]

So we may define the sequence $X^{(k)}=\left(1-\frac{1}{n}\right)^{-k}v^{(k)}$,
which is then martingale with respect to $v$, as
\begin{eqnarray*}
E[X^{(k+1)}|v^{(k)}] & = & \left(1-\frac{1}{n}\right)^{-k-1}E[v^{(k+1)}|v^{(k)}]\\
 & = & \left(1-\frac{1}{n}\right)^{-k}v^{(k)}\\
 & = & X^{(k)}.
\end{eqnarray*}

Now since $k$ is chosen in advance, stopping time theory gives that
$E[X^{(k)}]=E[X^{(0)}]$. So 
\[
E[\left(1-\frac{1}{n}\right)^{-k}v^{(k)}]=n,
\]
\[
\therefore E[v^{(k)}]=\left(1-\frac{1}{n}\right)^{k}n.
\]

By Assumption 1, the function can be at most minimized over the dimensions
seen up to step $k$. The seen dimensions contribute a value of $\frac{1}{4}$
and the unseen terms $\frac{1}{2}$ to the function. So we have that:
\begin{eqnarray*}
E[f(w^{(k)})]-f(w^{*}) & \geq & \frac{1}{4}\left(n-E[v^{(k)}]\right)+\frac{1}{2}E[v^{(k)}]-\frac{n}{4}\\
 & = & \frac{1}{4}E[v^{(k)}]\\
 & = & \left(1-\frac{1}{n}\right)^{k}\frac{n}{4}\\
 & = & \left(1-\frac{1}{n}\right)^{k}\left[f(w^{(0)})-f(w^{*})\right].
\end{eqnarray*}

\end{proof}

\subsection*{Minimization of non-strongly convex finite sums}

It is known that the class of convex, continuous \& differentiable
problems, with $L-$Lipschitz continuous derivatives $F_{L}^{1,1}(R^{m})$
, has the following lower complexity bound when $k<m$:
\[
f(x^{(k)})-f^{(k)}(x^{*})\geq\frac{L\left\Vert x^{(0)}-x^{*}\right\Vert ^{2}}{8(k+1)^{2}},
\]

which is proved via explicit construction of a worst-case function
where it holds with equality. Let this worst case function be denoted
$h^{(k)}$ at step $k$.

We will show that the same bound applies for the finite-sum case,
on a per pass equivalent basis, by a simple construction. 
\begin{thm}
The following lower bound holds for $k$ a multiple of $n$:
\end{thm}
\[
f(x^{(k)})-f^{(k)}(x^{*})\geq\frac{L\left\Vert x^{(0)}-x^{*}\right\Vert ^{2}}{8(\frac{k}{n}+1)^{2}},
\]

when $f$ is a finite sum of $n$ terms $f(x)=\frac{1}{n}\sum_{i}f_{i}(x)$,
with each $f_{i}\in F_{L}^{1,1}(R^{m})$, and with $m>kn$, under
the oracle model where the optimization method may choose the index
$i$ to access at each step.
\begin{proof}
Let $h_{i}$ be a copy of $h^{(k)}$ redefined to be on the subset
of dimensions $i+jn$, for $j=1\dots k$, or in other words, $h_{i}^{(k)}(x)=h^{(k)}([x_{i},x_{i+n},\dots x_{i+jn},\dots])$.
Then we will use:
\[
f^{(k)}(x)=\frac{1}{n}\sum_{i}h_{i}^{(k)}(x)
\]

as a worst case function for step $k$. 

Since the derivatives are orthogonal between $h_{i}$ and $h_{j}$
for $i\neq j$, by Assumption 1, the bound on $h_{i}^{(k)}(x^{(k)})-h_{i}^{(k)}(x^{*})$
depends only on the number of times the oracle has been invoked with
index $i$, for each $i$. Let this be denoted $c_{i}$. Then we have
that:
\[
f(x^{(k)})-f^{(k)}(x^{*})\geq\frac{L}{8n}\sum_{i}\frac{\left\Vert x^{(0)}-x^{*}\right\Vert _{(i)}^{2}}{(c_{i}+1)^{2}}.
\]

Where $\left\Vert \cdot\right\Vert _{(i)}^{2}$ is the norm on the
dimensions $i+jn$ for $j=1\dots k$. We can combine these norms into
a regular Euclidean norm:
\[
f(x^{(k)})-f^{(k)}(x^{*})\geq\frac{L\left\Vert x^{(0)}-x^{*}\right\Vert ^{2}}{8n}\sum_{i}\frac{1}{(c_{i}+1)^{2}}.
\]

Now notice that $\sum_{i}\frac{1}{(c_{i}+1)^{2}}$ under the constraint
$\sum c_{i}=k$ is minimized when each $c_{i}=\frac{k}{n}$. So we
have:

\begin{eqnarray*}
f(x^{(k)})-f^{(k)}(x^{*}) & \geq & \frac{L\left\Vert x^{(0)}-x^{*}\right\Vert ^{2}}{8n}\sum_{i}\frac{1}{(\frac{k}{n}+1)^{2}},\\
 &  & \frac{L\left\Vert x^{(0)}-x^{*}\right\Vert ^{2}}{8(\frac{k}{n}+1)^{2}},
\end{eqnarray*}

which is the same lower bound as for $k/n$ iterations of an optimization
method on $f$ directly.\end{proof}